\documentclass[journal]{IEEEtai}

\usepackage[colorlinks,urlcolor=blue,linkcolor=blue,citecolor=blue]{hyperref}
\let\labelindent\relax
\usepackage{color,array}

\usepackage{graphicx}
\usepackage{amssymb}

\usepackage{hyperref}
\usepackage{url}
\usepackage{algorithm}
\usepackage{algorithmicx}
\usepackage{algpseudocode}
\usepackage[inline]{enumitem}
\usepackage{multirow}
\usepackage{xcolor}
\usepackage{bigstrut}

\usepackage{diagbox}
\usepackage{graphicx}
\usepackage{caption}
\usepackage{subcaption}
\usepackage{enumitem}
\usepackage{booktabs}
\usepackage{makecell}
\usepackage{color}
\usepackage{amsmath,bm}
\usepackage{booktabs}
\usepackage{makecell}
\usepackage{color}
\usepackage{amssymb}
\usepackage{amsthm}
\usepackage{stmaryrd}

\begin{document}

\title{Mitigating Backdoors in Federated Learning with FLD}

\author{Yihang Lin, Zhiqian Wu, Yong Liao, Pengyuan Zhou}


\maketitle

\begin{abstract}
Federated learning allows clients to collaboratively train a global model without uploading raw data for privacy preservation. This feature, i.e., the inability to review participants' datasets, has recently been found responsible for federated learning's vulnerability in the face of backdoor attacks. Existing defense methods fall short from two perspectives: 1) they consider only very specific and limited attacker models and thus are unable to cope with advanced backdoor attacks, such as distributed backdoor attacks, which break down the global trigger into multiple distributed triggers.  2) they conduct detection based on model granularity thus their performance gets impacted by the model size. 

To address these challenges, we propose Federated Layer Detection (FLD), a novel model filtering approach to effectively defend against backdoor attacks. FLD examines the models on layer granularity to capture the complete model details and effectively detect potential backdoor models regardless of model size. We provide theoretical analysis and proof for the convergence of FLD. Extensive experiments demonstrate that FLD effectively mitigates state-of-the-art (SOTA) backdoor attacks with negligible impact on the accuracy of the primary task, outperforming SOTA defense methods.
\end{abstract} 

\begin{IEEEImpStatement}
Federated learning (FL), which enables distributed clients to collaboratively train a global model without data centralization, preserves user privacy and has become increasingly prominent. Despite its advancement, FL's vulnerability to adversarial attacks, particularly backdoor threats, has become a critical concern. In FL, the aggregation server does not have direct access to the data, which makes it harder to find these sneaky attacks. This makes it very hard to keep the model's integrity without losing accuracy. To address this, we introduce Federation Layer Detection (FLD), an innovative model filtering approach designed to counteract backdoor attacks in FL. Our experiments across various datasets demonstrate that FLD effectively mitigates such attacks with minimal impact on the main task's performance, highlighting its potential as an effective defense mechanism in secure FL systems.
\end{IEEEImpStatement}


\section{Introduction}

The rapidly developing artificial intelligence technologies have been applied to numerous fields, such as computer vision, natural language processing, data mining, and so on. 
Many AI services are supported by cloud services to collect big data from large numbers of distributed users. However, privacy protection regulations released in recent years, such as GDPR~\cite{GDPR}, CCPA~\cite{CCPA} and PDPA~\cite{PDPA}, have presented serious challenges for user data collection. Therefore, privacy-preserving data analysis technologies are critical. 


Federated learning~\cite{googlefederation}, proposed by Google in 2016, allows distributed clients to jointly train a global model without uploading user private data and thus has attracted wide attention. In federated learning, a central parameter server sends an initial global model to the clients. The clients then train local models using local datasets and send the updated model parameters to the server after training. The server updates the global model via parameter aggregation and distributes the updated global model to the participants in the next round. The process iterates until the model is convergent or the predefined period ends.


As a trade-off, federated learning limits the central server's access to users' data. However, this feature, as uncovered by recent studies, makes federated learning vulnerable to adversarial attacks, especially backdoor attacks~\cite{BadNets,Trojaning}. Backdoor attacks embed backdoor into the model without impacting the primary task performance, thus being much more stealthy than other attacks such as untargeted poisoning attacks~\cite{untarget}. 

The defenses against centralized backdoor attacks fall into three broad categories: pre-training, during-training, and post-training~\cite{Backdoorsurvey}.~
%
Pre-training requires pre-processing of training data and post-training requires model fine-tuning using users' datasets, both of which require access to local data and thus cannot be applied to federated learning. As a result, backdoor defenses for federated learning mostly use robustness aggregation or differential privacy perturbation during the training stage. However, the state-of-the-art~(SOTA) defense methods are either only applicable to independent and identically distributed~(i.i.d) datasets, or can only defend against conventional backdoor attacks but fail at defending SOTA attacks like DBA~\cite{xie2020dba} and A Little Is Enough~\cite{littleisenough}. 

To address these challenges, we propose Federated Layer Detection~(FLD), which measures the fine-grained layer-level differences across models. FLD consists of two major modules, namely Layer Scoring and Anomaly Detection, which can reliably identify and remove malicious clients to guarantee system performance.  Layer Scoring assigns an outlier score to each layer of the models uploaded by the clients based on the concept of \textit{density} instead of commonly used \textit{distance} which sometimes causes wrong detection results~\cite{cof}.  Anomaly Detection uses the median absolute deviation (MAD)~\cite{MAD} of layer scores to determine if a model is anomalous. FLD overcomes three major limitations of existing defense methods, namely, simplified assumptions of data distribution, degraded accuracy of the primary task, and only functioning against specific backdoor attacks. Our contributions are mainly threefold, as follows:
\setlist[itemize]{leftmargin=*}
\begin{itemize}
    \item We propose FLD, an innovative backdoor defense scheme for federated learning. FLD is the first fine-grained defense scheme that assesses models based on layer level, to our best knowledge. 
    \item Layer Scoring module captures the fine-grained model details to improve the generalizability of FLD to deeper models. Anomaly Detection employs MAD to avoid impactful mean shifts caused by extreme outlier scores of anomalous models. 
    As such, FLD can effectively detect potential backdoored models regardless of model size. 
    \item We theoretically prove the convergence guarantee of FLD in both i.i.d and non-i.i.d data distributions. We also prove the correctness of FLD in homomorphic encryption scenarios.
    \item Extensive experiments on several well-known datasets can show that FLD effectively defends against a wide range of SOTA federated backdoor attacks in different scenarios without compromising the accuracy of the primary task, demonstrating FLD's robustness and generalizability.
\end{itemize}

\section{RELATED WORK}
\subsection{Federated Learning}
Federated learning is a popular strategy for training models on distributed data while preserving client privacy. Assuming there are $N$ clients, $ C = \left \{ C_{1}, C_{2}, C_{3},\cdots, C_{N}  \right \} $, each of which has a local training dataset $D_{i}, i\in \{1,\cdots, N\}$, that can communicate with the central parameter server to collaboratively train a global model. The standard federated learning process generally follows these phases:
\setlist[enumerate]{leftmargin=*}
\begin{enumerate}
  \item  In round $t$, the central parameter server randomly selects $n$ clients that satisfy predefined conditions to participate in the training and sends these clients the latest global model $ G^{t-1}$.
  \item Each selected client $C_i$ executes $\tau _{i}$ iterations to update its local model: $w_{i}^{t}= w_{i}^{t-1}-\eta_{i}^t g_{i}\left (  w_{i}^{t-1};\xi_{i,k}^{t-1}  \right )$, where $\eta _{i}^t$ is the learning rate, $\xi_{i,k}^{t-1}$ is a batch of uniformly chosen data samples. After the local training, $C_i$ sends updated $w_{i}^{t}$ to the central parameter server. 
  \item The parameter server aggregates the received models to update the global model. We choose the classical FedAvg algorithm~\cite{fedavg} for aggregation: $ G^{t} = {\textstyle \sum_{i=1}^{n}} \frac{ m_{i} }{m} w_{i}^{t}$, where $m_{i} = \left \| D_{i}  \right \|,m= {\textstyle \sum_{i=1}^{n}m_{i} } $. Previous works~\cite{howtobackdoor,xie2020dba,krum} typically used the same weights $\left ( \frac{ m_{i} }{m} =\frac{1}{n}  \right ) $  to average client contributions. For simplicity, we follow this setup, i.e., $G^{t} = {\textstyle \sum_{i=1}^{n}} \frac{1}{n} w_{i}^{t}$.

\end{enumerate}

\subsection{Backdoor Attacks}
Backdoor attacks impact the model training with ``carefully crafted'' poisoned data (mixing triggers into a small portion of data) to get a backdoor-poisoned model. The corrupted model outputs original labels for clean samples but target labels for poisoned samples. Thus, it does not impact primary task performance and is difficult to detect. Besides conventional centralized backdoor attacks, distributed backdoor attacks have emerged in recent years. A formal description of the attacker's purpose is
\begin{equation}
\forall \left ( x,y \right ) \in D, f\left (  G,x\right ) = y \wedge f\left (  G,x^{*} \right ) = y_{backdoor},
\end{equation}
where $D$ is the clean dateset, $ y_{backdoor} $ is the attacker's target label, $x^{*}$ is the clean sample $x$ combined with backdoor triggers.\\
\textbf{Backdoor Attacks on Federated Learning} can be divided into two categories: data poisoning and model poisoning~\cite{federatedopen}:
\setlist[itemize]{leftmargin=*}
\begin{itemize}
  \item \textbf{Data poisoning}: The attacker can only modify the dataset of the compromised clients but cannot control their training process or modify the data uploaded by the clients to the parameter server. The common methods are label flipping (e.g., making that picture of a cat labeled as a dog) and adding triggers to the image samples (e.g., adding a hat to the face images). To avoid being detected, the attackers often control the Poisoned Data Rate (PDR) to restrict the poisoned model's deviation from the benign model. Let $D_{i}$ denote the poisoned training dataset of the compromised client $i$ and $D_{i}^{poi}$ denote the poisoned data, then the PDR of $ D_{i} $ is
\begin{equation}
 PDR=\frac{\left | D_{i}^{poi} \right | }{\left | D_{i} \right | }.  
 \end{equation}
  \item \textbf{Model poisoning} is more powerful than data poisoning because the attacker can manipulate the compromised client's training and directly modify their uploaded data to maximize the impact, driving the global model closer to the backdoor model without being noticed by the anomaly detection mechanism running in the parameter server. A classic model poisoning attack is the model-replacement attack~\cite{howtobackdoor} which scales the uploaded model parameters. 
Model poisoning attacks can be divided into two types, namely scaling and evasion.
\textbf{Scaling} drives the global model close to the backdoor model by scaling up the weights of the uploaded model.
\textbf{Evasion} constrains model variation during training to reduce the malicious model's deviation from the benign model in order to avoid anomaly detection. A common method is to modify the target loss function by adding an anomaly detection term $L_{ano}$ as follows:
\begin{equation}
L_{model} = \alpha L_{class}+\left ( 1-\alpha  \right ) L_{ano},
\end{equation}
where $L_{class}$ captures the accuracy of the primary and backdoor tasks, the hyperparameter $\alpha$ controls the importance of evading anomaly detection. $L_{ano}$ functions as the anomaly detection penalty, e.g., the Euclidean distance between the local models and global model.
Model poisoning can directly control the training process and modify the model weights. Therefore, model poisoning can bypass the anomaly detection mechanism and robust aggregation mechanism deployed on the parameter server. It scales the model weights to satisfy the bound defined by the anomaly detection mechanism.
\end{itemize}
\subsection{Backdoor Defenses}
Works on defending against backdoor attacks in federated learning can be broadly categorized into two directions: robust aggregation and anomaly model detection. Robust aggregation optimizes the aggregation function to mitigate the impact of contaminated updates sent by attackers. One common approach is to apply a threshold to limit the impact of updates from all clients on the global model, such as by constraining the l2 norm of the updates (referred to as clipping)~\cite{howtobackdoor,foolgold,flame,l21}. Other approaches explore new global model estimators, like Robust Federated Averaging (RFA)~\cite{RFA} and Trimmed Mean~\cite{Trimmed_Mean}, to enhance the robustness of the aggregation process. However, a major drawback of these approaches is that contaminated updates may still persist in the global model, resulting in reduced model accuracy and incomplete mitigation of backdoor effects. Additionally, applying update constraints to all clients, including benign ones, reduces the magnitude of updates and thus slows down the convergence. The second direction, anomaly model detection, aims to identify and remove malicious updates from the aggregation process. This task is challenging due to the non-independent and heterogeneous distribution of client data and the uncertainty of the number of malicious clients. Previous methods have typically utilized clustering directly for detection.

\section{Problem Setup}\label{sec:problem}
First, we describe the assumptions behind the convergence analysis of FLD, then we introduce the concepts that are vital to the algorithm design and analysis.

Let $F_i$ denotes the local model of the $i$-th client, $i=1,2,\cdots,N$. Let $F$ denotes the global model in the central parameter server. Suppose our models satisfy Lipschitz continuous gradient, we make Assumptions \ref{assumption1} and \ref{assumption2}.
\newtheorem{assumption}{Assumption}
\begin{assumption}\label{assumption1}
($L$-smooth). $F_1, \cdots,F_N$ are all $L$-smooth: 
\begin{equation*}
    \forall x, y,\quad F_i(x)\leq F_i(y)+(x-y)^{\mathsf{T}}\nabla F_i(y)+\frac{L}{2}||x-y||_2^2.
\end{equation*}
\end{assumption}
 
\begin{assumption}\label{assumption2}
($\mu$-strongly convex). $F_1,\cdots,F_N$ are all $\mu$-strongly convex:
\begin{equation*}
    \forall x, y,\quad F_i(x)\geq F_i(y)+(x-y)^{\mathsf{T}}\nabla F_i(y)+\frac{\mu}{2}||x-y||_2^2.
\end{equation*}
\end{assumption}
We also follow the assumption made by~\cite{stich2018sparsified,yu2019parallel,li2019convergence} as follows.
\begin{assumption}\label{assumption3}
The expected squared norm of stochastic gradients is uniformly bounded, i.e., $\exists U>0$, $\mathbb{E}||\nabla F_i(\cdot)||^2 \leq U^2$ for all $i=1,\cdots,N$.
\end{assumption}
We make Assumption \ref{assumption4} to bound the expectation of $||w_i^t||^2$, where $w_i^t$ denotes the parameters of $F_i$ in $t$-round.
\begin{assumption}\label{assumption4}
(Bounding the expectation of $|| w_i^t ||^2$). The expected squared norm of $i$-th client's local model parameters is bounded: $\exists M>0$, $\mathbb{E}||w_i^t||^2 \leq M^2$ for all $i=1,\cdots,N$ and $t=1,\cdots,T$.
\end{assumption}

\noindent\textbf{Adversary Capability.} Following previous works~\cite{howtobackdoor,krum,munoz2019byzantine,nguyen2020poisoning}, we assume the attacker controls a portion (less than 50\%) of the clients, called ``compromised clients''. We assume that the attacker can possess the strongest attack capability, i.e., both data poisoning and model poisoning. The attacker can send arbitrary gradient contributions to the aggregator in any iteration according to its observation of the global model state. Compromised clients can collude in an intrinsic and coordinated fashion by sharing states and updates with each other. The attacker has no control over the benign clients or the server's aggregation process.

\noindent\textbf{Adversarial Target.} To ensure the effectiveness of the attack, the attack should: 1) be stealthy, i.e., injecting the backdoor should not lead to accuracy fluctuation of the model's primary task 2) have a high attack success rate (ASR), i.e., the success rate that the model identifies the samples containing backdoor triggers as target label should be as high as possible.

The adversarial objective of the compromised client $i$ in round $t$ can be denoted as follows:
\begin{equation}
\begin{split}
w_{i}^ {t} =&arg\max_{w_{i}} ( \sum _{j \in  D_ {i}^ {poi}}  P[  G^ {t+1} (R(  x_ {j}^ {i}  ,  \phi  ))=   y_{backdoor}  ]+\\  &\sum_{j \in  D_ {i}^ {cle}}   P[  G^ {t+1}  (  x_ {j}^ {i}  )=  y_ {j}^ {i}  ]).
\end{split}
\end{equation}
where $D_ {i}^ {poi}$ and $D_ {i}^ {cle}$ denote the poisoned dataset and clean dataset respectively, $D_ {i}^ {poi}\cap D_ {i}^ {cle}=\emptyset $ and $D_ {i}^ {poi}\cup  D_ {i}^ {cle}=D_{i} $. Function $R$ transforms clean data $ x_ {j}^ {i}$ into poisoned data $R(  x_ {j}^ {i}, \phi)$ by embedding the trigger $\phi$. An attacker trains its local model to find the optimal parameters $w_{i}^ {t}$ so that $G^ {t+1}$ identifies poisoned data $R(  x_ {j}^ {i},  \phi  )$ as backdoor target label $y_{backdoor}$ and identifies clean data $ x_ {j}^ {i}$ as ground truth label $y_ {j}^ {i}$.

\noindent\textbf{Defense Goal}.Our defense goal is to mitigate backdoor attacks in federated learning. Specifically, an efficient defense algorithm needs to: 1) ensure the performance of the global model in terms of the main task's accuracy, 2) minimize the possibility of outputting backdoor target labels, and 3) defend against a wide range of SOTA federated backdoor attacks without prior knowledge of the proportion of compromised clients.

\section{Methodology}


\begin{figure*}[!t]
\centering
\includegraphics[width=.8\linewidth]{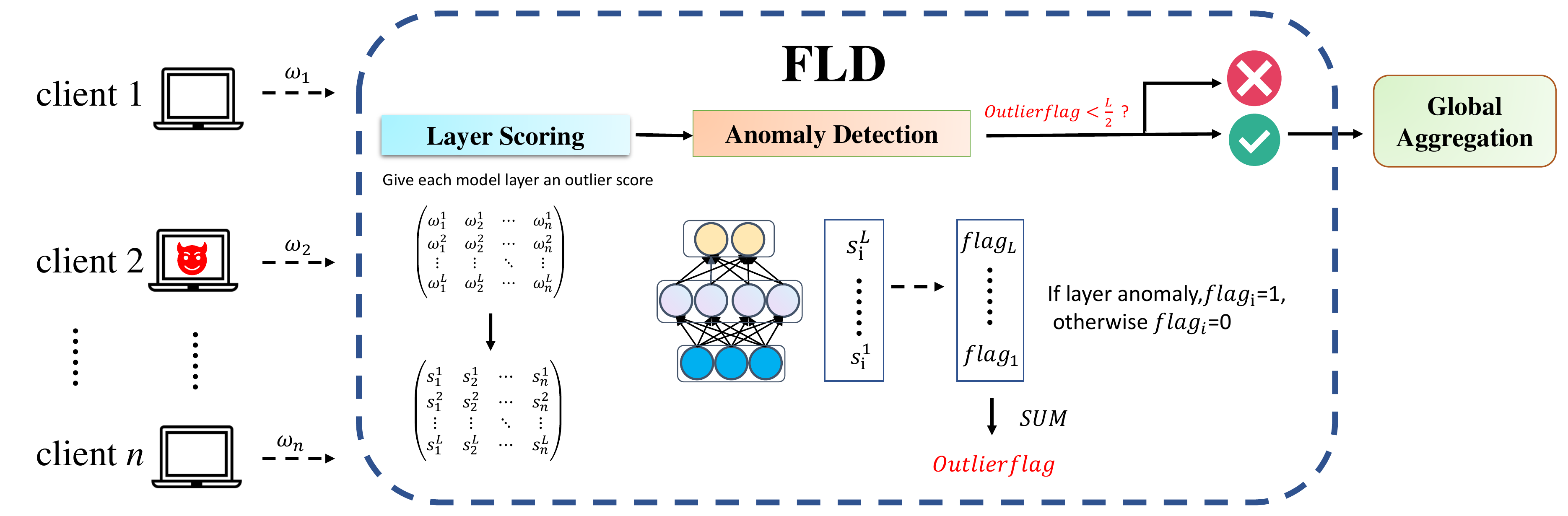}
\caption{Procedure of FLD}
\label{Overview}
\end{figure*}

\subsection{Motivation}
According to the attack analysis in constrain-and-scale~\cite{howtobackdoor}, the federated backdoor attacks are divided into two attack scenarios: single-shot attacks (\textit{Attack A-S}) and multi-shot attacks (\textit{Attack A-M})~\cite{xie2020dba,am1}.
\setlist[itemize]{leftmargin=*}

\noindent\textbf{\textit{Attack A-S}}: the attacker successfully embeds its backdoor trigger in only one round. The attacker performs parameter scaling on the compromised clients' updates to substitute the global model $G^{t}$ with a backdoor model $X$ in Equation~\ref{eq:backdooreq1}: 
\begin{equation}
\label{eq:backdooreq1}
X = {\sum_{i=1}^{n}} \frac{1}{n} w_{i}^{t}.
\end{equation}
To achieve this goal, the attacker can scale the model parameters as follows:
\begin{equation}
\begin{split}\label{eq:backdooreq2}
\tilde{w} _{n}^{t} 
& = nX-\sum_{i=1}^{n-1} w_{i}^{t}\approx nX-\sum_{i=1}^{n-1} G ^{t-1}\\
& = n\left ( X-G ^{t-1} \right )+G ^{t-1}.
\end{split}
\end{equation}
As the global model converges, $w_{i}^{t}\approx  G ^{t-1}$. In other words, the attacker scales up the model weights $X$ by $n$ to prevent the malicious updates from being mitigated by the aggregation.

\noindent\textbf{\textit{Attack A-M}} lets compromised clients accumulate malicious updates over multiple rounds, instead of directly scaling the uploaded parameters, to avoid being detected by the defense algorithm. 
We have thoroughly reviewed the SOTA backdoor defense works in federated learning~\cite{foolgold,krum,Trimmed_Mean,Bulyan,RFA,canyou} and found that existing methods focus on defending against \textit{Attack A-S} while overlooking \textit{Attack A-M}. Unsurprisingly, we found through empirical experiments that SOTA defense algorithms fail at defending against {\textit{Attack A-M}.

To address the challenges mentioned above, we propose Federated Layer Detection~(FLD), an innovative defense method for effectively detecting backdoor attacks in federated learning that overcomes the deficiencies of previous works. As depicted in Fig.~\ref{Overview}, FLD consists of two components, namely Layer Scoring and Anomaly Detection. Layer Scoring assigns scores to the local models uploaded by the clients according to the concept of isolation. Anomaly Detection checks the outlier scores given by Layer Scoring to determine if an uploaded model is compromised. The overall process is as follows: 1) The server receives the local models from the clients participating in the current round. 2) The server assigns each layer of each model an outlier score using Layer Scoring. 3) The server labels each layer as abnormal or not according to its outlier score, and excludes from the aggregation the anomalous models that contain more anomalous layers than the predefined threshold, as summarized in Algorithm~\ref{algorithm1}.
\begin{algorithm}
    \caption{Overview}\label{algorithm1}
    \begin{algorithmic} [1]
        \algrenewcommand\algorithmicrequire{\textbf{Input:}}
        \Require Set of clients $ C = \left \{ C_{1},C_{2},C_{3},\cdots,C_{N}  \right \} $, local datasets $ D = \left \{ D_{1},D_{2},D_{3},\cdots,D_{N} \right \}$, the number of training iterations $T$, the percentage of participating clients per round $K$.
        \algrenewcommand\algorithmicrequire{\textbf{Output:}} 
        \Require Global model $G^T $
        \State Initialize the global model $G^0 $
        \For { $ t\ in \left [ 1,T \right ] $ }
        \State $ n\gets \max \left ( K\cdot N,1  \right ) $ 
        \State $ C^{t} \gets $ (random set of $n$ clients )
        \For {each client $ i\in C^{t} $ in parallel} 
        \State The server sends $G^{t-1}$ to client $i$
        \State $w_{i}^{t} \gets $ ClientUpdate$\left ( D_i, G^{t-1}  \right) $ 
        \State Client $i$ sends $w_{i}^{t}$ back to the server
        \EndFor
        \State $\left(S_{1} ,\cdots,S_{n}\right)\gets Layer Scoring\left(w_{1}^{t},\cdots,w_{n}^{t}\right)$
        \State $C_{b}^{t}\gets Anomaly Detection \left(S_{1} ,\cdots,S_{n}\right)$
        \State  $m\gets len\left ( C_b^t \right ) $
        \State  $  G^{t}=  \sum_{i\in  C_b}^{} \frac{1 }{m} w_{i}^{t}$
        \EndFor
    \end{algorithmic}
\end{algorithm}


%
\subsection{Layer Scoring}
In round $t$, the parameter server sends the global model $G^{t-1}$ to the selected clients $i\in C^t$, each of which trains $G^{t-1} $ using its local data $D_{i}$ and sends the model parameters $w_{i}^{t}$ back to the server after local training is completed.

Existing backdoor defense methods assess uploaded models based on either similarity or distance metrics. They usually flatten the parameters of each model layer and then stitch them into a vector to perform the assessment. However, different layers of the neural network have heterogeneous parameter value distribution spaces due to their different functions. For example, in a CNN, the lower layers learn to detect simple features such as stripes, the middle layers learn to detect a part of an object, and the higher layers learn to detect a concept (e.g., a person)~\cite{distribution}. As a result, directly flattening and splicing the parameters of each layer easily leads to the loss of important information and hence the escape of malicious models. Therefore, finer-grained detection is demanded.

To address this issue, we have devised a hierarchical detection method, Layer Scoring, to measure fine-grained model differences, as shown in Algorithm~\ref{algorithm2}. Layer Scoring examines and assigns an outlier score to each layer of the uploaded models in turn. To provide accurate scores, the outlier scoring method is crucial and faces the following challenges:

\noindent\textbf{C1}: \textbf{The proportion of compromised clients is unknown.} Many existing works on outlier detection reply to the impractical assumption of knowing the proportion of compromised clients in advance, which severely limits their applicability in reality. To address this limitation, in this work, we propose an algorithm without requiring such prior knowledge. As such, conventional outlier detection methods such as K-Nearest Neighbors (KNN) and One-Class SVM, which rely on prior knowledge of the number of neighbors, are not feasible. 


\noindent\textbf{C2}: \textbf{Identifying backdoored models in dynamic scenarios.} In each round, the number of injected backdoors is unknown and may vary. Hence, it is important to have a stable backdoored model identification method that can effectively handle dynamic attacks. Otherwise, many false positives may be generated, failing the backdoor defenses and impacting the main task's accuracy.


To address the above challenges, we chose Connectivity-based Outlier Factor~(COF)~\cite{cof} as our outlier detection algorithm. COF is a density-based outlier detection algorithm that measures the degree of connectivity of a data point to its neighboring points. COF calculates the outlier score of each data point by comparing its average reachability distance to that of its neighbors. COF is advantageous over other distance-based outlier detection algorithms as it is less sensitive to the number of dimensions of the data and can effectively detect outliers in high-dimensional data. It is also able to detect outliers in non-uniform density data sets and is less affected by the presence of noise in the data. Additionally, COF does not require any assumptions about the underlying data distribution, making it more robust to different types of data. Therefore, COF is a better choice for identifying backdoored models in a dynamic federated learning setting where the proportion of compromised clients is unknown and may vary over time.



\begin{algorithm}
    \caption{Layer Scoring}\label{algorithm2}
    \begin{algorithmic} [1]
        \algrenewcommand\algorithmicrequire{\textbf{Input:}}
        \Require The local model $w_{i}$ uploaded by each client $ i\in C^{t} $
        \algrenewcommand\algorithmicrequire{\textbf{Output:}} 
        \Require  The set of Layer Scoring $S_{i}$ for each client $ i\in C^{t}$ 
        \State \textbf{initialize }  $n\gets len\left ( C^{t} \right ) $
        \For {$layer\ j\ in \left [ 1,total \right ]$ } \Comment{$total$ is the number of layers of the model }
        \State  $\left ( s_{1}^{j} ,\cdots,s_{n}^{j} \right ) \gets  COF\left ( w_{1}^{j} ,\cdots,w_{n}^{j} \right ) $
        \For {$  i \in \left [ n \right ]$ } 
        \State Add $s_{i}^{j}$ to the set of Layer Scoring $S_{i}$
        \EndFor
        \EndFor
        \State return $S_{1} ,\cdots,S_{n}$
    \end{algorithmic}
\end{algorithm}

\subsection{Anomaly Detection}
Layer Scoring assigns each layer of each local model an outlier score. Then, Anomaly Detection uses the scores to identify the anomalous clients to safeguard the model from backdoor attacks. Anomaly Detection checks the scores \textit{layer by layer} and increments a model's flag count by one upon finding an abnormal layer score. In the end, the clients with the higher flag counts are marked as anomalies. In this paper, we mark clients with more than 50\% of the layer count as anomalies. The algorithm for determining layer anomalies needs to be carefully designed to achieve the three \textit{defense goals} as mentioned in Section~\ref{sec:problem}.

We follow the common assumption that less than 50\% of clients are compromised. We argue that the commonly employed Three Sigma Rule~\cite{ThreeSigmaRule} and Z-score~\cite{Zscore}} can not identify anomalous clients well, because these algorithms assess clients using the mean value, which can be strongly influenced and shifted towards the location of the outliers in the presence of extreme outliers, resulting in failed outlier identifications.
To solve this problem, we use MAD for anomaly detection because: 
\begin{enumerate*}
  \item \textbf{it tolerates extreme values} since MAD uses the median which is not affected by extreme values, and 
  \item \textbf{it can be applied to any data distribution}, unlike Three Sigma Rule and Z-score which are only applicable to normally distributed data.
\end{enumerate*}
The Anomaly Detection processes include: 
\begin{enumerate*}
  \item Calculate the median of all the outlier scores.
  \item Calculate the absolute deviation value of the outlier scores from the median.
  \item Assign the median of all the absolute deviation values to MAD.
  \item A layer whose outlier score deviates from the median by larger than $\mu$ MAD is classified as anomalous and the model's flag ($Outlierflag$) is incremented by one. $\mu$ is the hyperparameter which we set to 3 by default in the experiments.
\end{enumerate*}
In each round, the server gets all the uploaded models and checks all their layers to get $ Outlierflag_{i},\forall i\in \left [ n \right ] $. FLD classifies the models of which at least half of the layers are marked as anomalies as anomalous models and aggregates only the other models that are classified as benign models. 
\begin{algorithm}
    \caption{Anomaly Detection}\label{algorithm3}
    \begin{algorithmic} [1]
        \algrenewcommand\algorithmicrequire{\textbf{Input:}}
        \Require: The set of Layer Scoring from each client $ i\in C_{t}$ are regarded as $S_{i}$
        \algrenewcommand\algorithmicrequire{\textbf{Output:}} 
        \Require  The benign clients set $C_b^t$
        \State \textbf{initialize }  $n\gets len\left ( C^{t} \right ) $
        \State \textbf{initialize }  $Outlierflag_{i}  \gets 0, \forall i\in \left [ n \right ] $
        
        \For {$layer\ j\ in \left [ 1,total \right ]$ } \Comment{$total$ is the number of layers of the model }
        \State  $ Me \gets  MEDIAN\left ( S_{1}^{j} ,\cdots,S_{n}^{j} \right ) $
        \State  $MAD\gets MEDIAN(\left | S_{1}^{j}-Me \right |, \cdots, \left | S_{n}^{j}-Me \right |)$
        \For {$  i \in \left [ n \right ]$ } 
        \State $ flag1\gets \left ( S_{i}^{j} >= Me + \mu *MAD\right  ) ?1:0$
        \State $ flag2\gets \left ( S_{i}^{j} <= Me - \mu *MAD\right  ) ?1:0$
        \State $ Outlierflag_{i}\gets Outlierflag_{i}+flag1+flag2$
        \EndFor
        \EndFor
        \For {$  i \in \left [ n \right ]$ }
        \If{$Outlierflag_{i}<total/2$}
        \State Add $i$ to the benign clients set $C_b^t$
        \EndIf
        \EndFor
        \State return $C_b^t$
    \end{algorithmic}
\end{algorithm}

\subsection{Private FLD}
Many attacks on federated learning have been proposed besides backdoor attacks, such as membership inference attack and attribute inference attack. These attacks all demonstrate the necessity of enhancing the privacy protection of federated learning to prohibit access to local model plaintext updates. In general, there are two approaches to protect the privacy of customer data: differential privacy and encryption techniques such as homomorphic encryption~\cite{hom} or multi-party secure computation~\cite{mpc}. Differential privacy is a statistical and simple-to-implement method, but with impacts on the model performance, while encryption provides strong privacy guarantees and protection, but at the cost of reduced efficiency.
Specifically, homomorphic encryption is a cryptographic primitive that allows computations to be performed on encrypted data without revealing the underlying plaintext. The basic idea is to encrypt the plaintext first to obtain the ciphertext and continue the calculation operation on the ciphertext, decrypt the final ciphertext result to obtain the plaintext, to keep the result consistent with the calculation on the plaintext. For example, Paillier cryptosystem is a representative additive homomorphic encryption that has been commonly used in federated learning. It has the following two homomorphic properties:
\setlist[itemize]{leftmargin=*}
\begin{itemize}
\item \textbf{Homomorphic addition of plaintexts}: $\llbracket{ x_{1}\rrbracket}\cdot \llbracket{ x_{2}\rrbracket}= \llbracket{x_{1}+x_{2}\rrbracket}$, where $ x_{1}$ and $ x_{2}$ represent  plaintexts, $\llbracket{~\rrbracket}$ 
is an encryption operation.
\item \textbf{Homomorphic multiplication of plaintexts}: $ \llbracket{ x\rrbracket}^{r}= \llbracket{ r\cdot x\rrbracket} $, where  $ x$  represents  plaintext, $\llbracket{~\rrbracket}$ 
is an encryption operation, $r$ is  a constant.
\end{itemize}
Next we illustrate the applicability of FLD in federated learning homomorphic encryption scenarios.
First, we follow the federated setup of~\cite{privacyfl}:
\setlist[itemize]{leftmargin=*}
\begin{itemize}
\item \textbf{Server} is responsible for receiving the gradients submitted by all participants and conducting aggregation to obtain a new global model.
\item \textbf{Cloud Platform (CP)} performs homomorphic encryption calculations together with the server. The CP holds a \textit{(private-key, public-key)} pair generated by a trusted authority for encryption and decryption.
\end{itemize}
Our algorithm is summarized in Algorithm~\ref{algorithm4}.\\
\begin{algorithm}[t!]
    \caption{Private FLD}\label{algorithm4}
    \begin{algorithmic} [1]
        \algrenewcommand\algorithmicrequire{\textbf{Input:}}
        \Require: The local model $\llbracket{w_{i}}\rrbracket$ uploaded by each client $ i\in C^{t} $
        \algrenewcommand\algorithmicrequire{\textbf{Output:}} 
        \Require  The set of Layer Scoring $S_{i}$ for each client $ i\in C^{t}$ 
        \algrenewcommand\algorithmicrequire{\textbf{Server:}}
        \Require
        \State  Randomly select m nonzero integer $r_i $ for j in [1,m]
        \Comment{m is the length of $w_{i}$ }
        \For {$  j \ in \left [ 1,m \right ]$ } \Comment{$n$ is  }
        \State $ c_{ij} \gets \llbracket{\omega_{ij} }\rrbracket\cdot \llbracket{r_{j} }\rrbracket  $
        \EndFor
        \State send $\left \{  c_{ij}  \right \} _{j=1}^{j=n} $ to CP
    \end{algorithmic}
    \begin{algorithmic}[1]
    \algrenewcommand\algorithmicrequire{\textbf{CP:}}
    \Require:
    \For {$  j \ in \left [ 1,m \right ]$ } \Comment{$n$ is  }
        \State  $ \omega _{ij}^{'} \gets Dec(sk_{c},c_{ij} ) $
    \EndFor
    \State $\left(S_{1} ,\cdots,S_{n}\right)\gets Layer Scoring\left(w_{1}^{'},\cdots,w_{n}^{'}\right)$
    \State Send $\left(S_{1} ,\cdots,S_{n}\right)$ to PS
    \end{algorithmic}
\end{algorithm}
Correctness: To ensure that FLD can effectively identify malicious gradients, we need to prove that homomorphic encryption does not affect the calculation of COF anomaly detection.
According to the properties of homomorphic encryption, we have 
\begin{equation}
    \begin{split}\label{hm}
	  c_{ij}
	& =  \llbracket{\omega_{ij} }\rrbracket\cdot \llbracket{r_{j} }\rrbracket \\
	& =  \llbracket{\omega_{ij} } + {r_{j} }\rrbracket .
    \end{split}
\end{equation}
so $\omega _{ij}^{'} = \omega_{ij}  + r_{j} $, for $\omega _{x}^{'} $ and $\omega _{y}^{'} $ the Euclidean distance is
    \begin{equation}
    \begin{split}\label{hm}
	  \left \| \omega _{x}^{'}-\omega _{y}^{'} \right \| 
        & = \sqrt{\sum_{j=1}^{n}{\left( \omega _{xj}^{'}-\omega _{yj}^{'} \right)^{2} }}\\
	& =  \sqrt{\sum_{j=1}^{n}{\left( \omega_{xj}  + r_{j}- (\omega_{yj}  + r_{j}) \right)^{2} }} \\
        & = \sqrt{\sum_{j=1}^{n}{\left( \omega _{xj}-\omega _{yj} \right)^{2} }}\\
        & = \left \| \omega _{x}-\omega _{y} \right \|.
    \end{split}
\end{equation}

When the distance metric is Euclidean distance, the COF anomaly detection algorithm can still function in the homomorphic encryption scenario and the results are consistent with the plaintext.

\subsection{Convergence Analysis}
To analyze the convergence of FLD, we propose the theorem of convergence and prove it.


 


\newtheorem{thm}{Theorem}
\begin{thm}\label{thm1}
Let Assumptions \ref{assumption1} to \ref{assumption4} hold and $L$, $\mu$, $U$, $M$ be defined therein. Choose the learning rate $\eta^t=\frac{\theta}{t+\epsilon}$, $ \epsilon>0$, $\theta > \frac{1}{\mu}$, we define $\lambda=\max\{\frac{\theta A}{\theta \mu -1}, (\epsilon+1)Z_1\}$. Then FLD satisfies 
\begin{equation}
    \begin{split}
        \mathbb{E}[F(G^t)]-F^*
	\leq \frac{L}{2} Z_t 
	\leq \frac{L}{2}\frac{\lambda}{(t+\epsilon )^{\frac{1}{2}}}
	\stackrel{t \to \infty}{\longrightarrow}0,
    \end{split}
\end{equation}
where 
\begin{equation}
    \begin{split}
        & A=4U^2+M^2+2\Gamma, \\
        & Z_t=\mathbb{E}||G^t-G^*||^2.
    \end{split}
\end{equation}
\end{thm}

\begin{proof}
Let $G^{t+1}$ denote the global model's parameters in the central server in $(t+1)$-round and $G^*$ be the optimal parameters in the central server. Additionally, $g^t=\sum\limits_{i\in C_b^t}p_i\nabla F_i(w_i^t,\xi_i^t)$, where $g^t$ denotes the gradient updates uploaded by the clients in $t$-round and $p_i$ denotes the weight of the $i$-client's gradient during aggregation. $\bar{g^t}=\sum\limits_{i\in C_b^t}p_i\nabla F_i(w_i^t)$ and $G^{t+1}=G^t-\eta^t g^t$, where $C_b^t$ denotes the collection of benign clients chosen by FLD in $t$-round. Then, we have 
\begin{equation}
    \begin{split}\label{ineq1}
	||G^{t+1}-G^*||^2
	& = ||G^t-\eta^t g^t-G^*-\eta^t \bar{g^t} + \eta^t \bar{g^t}|| \\
	& = \underbrace{||G^t-G^*-\eta^t \bar{g^t}||^2}_{P_1} \\
        & +\underbrace{2\eta^t<G^t-G^*-\eta^t \bar{g^t},\bar{g^t}-g^t>}_{P_2} \\
	& +(\eta^t)^2||\bar{g^t}-g^t||^2.
    \end{split}
\end{equation}

Since $\mathbb{E}g^t=\bar{g^t}$, we see $\mathbb{E}P_2=0$. Now we split $P_1$ into three terms:   
\begin{equation}
    \begin{split}\label{ineq2}
	P_1
	& = ||G^t-G^*-\eta^t \bar{g^t}||^2 \\
	& = ||G^t-G^*||^2\underbrace{-2\eta^t<G^t-G^*,\bar{g^t}>}_{P_3}+\underbrace{(\eta^t)^2||\bar{g^t}||^2}_{P_4}.
    \end{split}
\end{equation}

Focusing on the last term in the above equation, according to Assumption \ref{assumption3}, we have 
\begin{equation*}
\begin{split}
    \mathbb{E}P_4
 & =\mathbb{E}[(\eta^t)^2||\bar{g^t}||^2] \\
 & \leq (\eta^t)^2\sum\limits_{i\in C_b^t}p_i^2\mathbb{E}||\nabla F_i(w_i^t)||^2 \\
 & \leq (\eta^t)^2 U^2 .  
\end{split}
\end{equation*}

Consider $P_3$, it follows:
\begin{equation}
    \begin{split}\label{ineq3}
	P_3
	& = -2\eta^t<G^t-G^*,\bar{g^t}> \\
	& = -2\eta^t\sum\limits_{i\in C_b^t}p_i<G^t-w_i^t,\nabla F_i(w_i^t)>\\
	& -2\eta^t\sum\limits_{i\in C_b^t}p_i<w_i^t-G^*,\nabla F_i(w_i^t)>.
    \end{split} 
\end{equation}

It is well known that $-2ab\leq a^2+b^2$, so
\begin{equation}
    \begin{split}\label{ineq4}
    & -2<G^t-w_i^t,\nabla F_i(w_i^t)>\\ 
    & \leq ||G^t-w_i^t||^2+||\nabla F_i(w_i^t)||^2 .
    \end{split}
\end{equation}

According to Assumption \ref{assumption2}, it follows:
\begin{equation}
    \begin{split}\label{ineq5}
    & -<w_i^t-G^*,\nabla F_i(w_i^t)> \\
    & \leq -(F_i(w_i^t)-F_i(G^*)) -\frac{\mu}{2}||w_i^t-G^*||^2 .
    \end{split}
\end{equation}

Use Equation~\ref{ineq2} and Inequalities~\ref{ineq3},~\ref{ineq4},~\ref{ineq5}, we obtain the following formula
\begin{equation*}
    \begin{split}
	   P_1  
        & = ||G^t-G^*-\eta^t \bar{g^t}||^2 \\
	& \leq ||G^t-G^*||^2+(\eta^t)^2||\nabla F_i(w_i^t)||^2 \\
        & +\eta^t\sum\limits_{i \in C_b^t}p_i(||G^t-w_i^t||^2
         +||\nabla F_i(w_i^t)||^2) \\
	& - 2\eta^t\sum\limits_{i \in C_b^t}p_i(F_i(w_i^t)-F_i(G^*)+\frac{\mu}{2}||w_i^t-G^*||^2)\\
	& \leq (1-\eta^t\mu)||G^t-G^*||^2+((\eta^t)^2+\eta^t)||\nabla F_i(w_i^t)||^2\\
	& + \eta^t\sum\limits_{i \in C_b^t}p_i||G^t-w_i^t||^2 \\
        & \underbrace{- 2\eta^t\sum\limits_{i \in C_b^t}p_i(F_i(w_i^t)-F_i(G^*))}_{P_5}.
    \end{split} 
\end{equation*}
Motivated by ~\cite{li2019convergence}, we define $\Gamma=F^*-\sum\limits_{i\in C_b^t}p_iF_i^*$. $\Gamma$ is used to measure the degree of heterogeneity between the local models and the global model, in i.i.d data distributions, $\mathbb{E}\Gamma=0$. We have 
\begin{equation*}
    \begin{split}
	P_5 
	& = - 2\eta^t\sum\limits_{i \in C_b^t}p_i(F_i(w_i^t)-F_i(G^*))\\
	& = - 2\eta^t\sum\limits_{i \in C_b^t}p_i(F_i(w_i^t)-F_i^*+F_i^*-F_i(G^*))\\
	& \leq  2\eta^t\sum\limits_{i \in C_b^t}p_i(F^*-F_i^*)=2\eta^t \Gamma, 
    \end{split} 
\end{equation*}
Hence,
\begin{equation*}
    \begin{split}
	P_1
	& \leq (1-\eta^t\mu)||G^t-G^*||^2+((\eta^t)^2+\eta^t)||\nabla F_i(w_i^t)||^2\\
	& + \eta^t\sum\limits_{i \in C_b^t}p_i||G^t-w_i^t||^2 +2\eta^t \Gamma.
    \end{split} 
\end{equation*}

Utilize the above results, we have
\begin{equation}
    \begin{split}\label{ineq6}
	\mathbb{E}|G^{t+1}-G^*||^2
	& \leq (1-\eta^t\mu)\mathbb{E}||G^t-G^*||^2 \\
        & + ((\eta^t)^2+\eta^t)\mathbb{E}||\nabla F_i(w_i^t)||^2\\
	& + \eta^t\sum\limits_{i \in C_b^t}p_i\mathbb{E}||G^t-w_i^t||^2 +2\eta^t \Gamma \\
        & + (\eta^t)^2\mathbb{E}||\bar{g^t}-g^t||^2.
    \end{split}
\end{equation}

According to Assumption \ref{assumption3}, it follows:
\begin{equation}
    \begin{split}\label{ineq7}
	\mathbb{E}||g^t-\bar{g^t}||^2
	& = \mathbb{E}||\sum\limits_{i \in C_b^t}p_i\nabla F_i(w_i^t,\xi_i^t)-\nabla F_i(w_i^t)||^2 \\
	& \leq \sum\limits_{i \in C_b^t}p_i^2 (\mathbb{E}||\nabla F_i(w_i^t,\xi_i^t)||^2 \\
        & +\mathbb{E}||\nabla F_i(w_i^t)||^2) \\
	&\leq 2\sum\limits_{i \in C_b^t}p_i^2 U^2.
    \end{split} 
\end{equation}

According to Assumption \ref{assumption4}, it follows:
\begin{equation}\label{ineq8}
    \begin{split}
	\sum\limits_{i \in C_b^t}p_i||G^t-w_i^t||^2
	& = \sum\limits_{i \in C_b^t}p_i||\sum\limits_{i \in C_b^t}p_i w_i^t-w_i^t||^2 \\
	& \leq \sum\limits_{i \in C_b^t}p_i || w_i^t||^2 \\
	& \leq M^2.
    \end{split}
\end{equation}

So far, we have all the preparations ready to prove the final conclusion. Let  $Z_t=\mathbb{E}||G^t-G^*||^2$, $\eta^t=\frac{\theta}{t+\epsilon}$, $\epsilon>0$, $\theta > \frac{1}{\mu}$, $\lambda=\max\{\frac{\theta A}{\theta \mu -1}, (\epsilon+1)Z_1\}$, our goal of proving $Z_t \leq \frac{\lambda}{(t+\epsilon)^{\frac{1}{2}}}$ can be achieved as follows.
\newline
For $t=1$, it holds. Suppose that the conclusion establishes for some t and use Inequalities~\ref{ineq6},~\ref{ineq7},~\ref{ineq8}, we have $Z_{t+1}$ as follows: 
\begin{equation}
    \begin{split}
        Z_{t+1}
	& \leq (1-\eta^t\mu) Z_{t}+((\eta^t)^2+\eta^t)U^2 + \eta^t M^2 \\
        &  + 2 (\eta^t)^2 \sum\limits_{i \in C_b^t}p_i^2 U^2+2\eta^t \Gamma \\
	& \leq (1-\eta^t\mu)Z_{t} + \eta^t A \\
	& = \frac{(t+\epsilon)^{\frac{1}{2}}-1}{(t+\epsilon)}\lambda+(\frac{\theta A}{t+\epsilon}-\frac{\theta \mu -1}{t+\epsilon}\lambda) \\
	& \leq \frac{\lambda}{(t+\epsilon +1)^{\frac{1}{2}}},
    \end{split}
\end{equation}
where $A=4U^2+M^2+2\Gamma$.
Then, from Assumption \ref{assumption1}, we get
\begin{equation}
    \begin{split}
        \mathbb{E}[F(G^t)]-F^*
	\leq \frac{L}{2} Z_t 
	\leq \frac{L}{2}\frac{\lambda}{(t+\epsilon )^{\frac{1}{2}}}
	\stackrel{t \to \infty}{\longrightarrow}0.
    \end{split}
\end{equation}
\end{proof}

\section{Experiments}
\subsection{Experiments Setup}
\textbf{Datasets and Models.}
To evaluate the effectiveness of FLD, we tested two different federated learning use cases, namely word prediction used by~\cite{howtobackdoor}, and image classification used by DBA~\cite{xie2020dba}.
\setlist[itemize]{leftmargin=*}
\begin{itemize}
  \item \textbf{Word Prediction:} We followed the experiment setup of~\cite{howtobackdoor}. We use the Reddit public dataset from November 2017 and filter out users with fewer than 150 or more than 500 posts. We assume that every remaining Reddit user is a participant in federated learning and treat each post as a sentence in the training data and use a model consisting of two LSTM layers and a linear output layer.
  \item \textbf{Image Classification:} We conducted experiments on three classic image datasets: MNIST, CIFAR10, and Tiny-imagenet. MNIST is a classic handwritten digital image set containing 60,000 training examples and 10,000 test examples. CIFAR10 consists of 10 classes of 32x32 color images, including a total of 50,000 training examples and 10,000 test examples. The Tiny-imagenet~\cite{xie2020dba} consists of 200 classes each of which has 500 training images and 50 test images. To simulate the non-i.i.d environment, we divide the datasets using Dirichlet distribution, a commonly used prior distribution in Bayesian statistics~\cite{NONIID}.
\end{itemize}

\begin{table}[htb]   
\begin{center}   
\caption{Dataset description}  
\label{table:dataset} 
\small 
\begin{tabular}{c|c|c}   
\hline   \textbf{Field} &\textbf{Datasets}   & \textbf{Model}\\  
\hline   NLP  & Reddit &   2-layer LSTM \\
\hline 
\multirow{3}*{Image Classification}  &MNIST   &  2 conv and 2 fc    \\ 
~  &  CIFAR10   &  lightweight Resnet-18 \\  
~  &  Tiny-imagenet   &  Resnet-18 \\     
\hline 
\end{tabular}   
\end{center}   
\end{table}


\noindent\textbf{Baselines.} We choose the following defense approaches as baselines: FoolsGold~\cite{foolgold}, Robust Federated Aggregation (RFA)~\cite{RFA}, Differential Privacy (DP)~\cite{canyou}, Krum~\cite{krum}, Trimmed Mean~\cite{Trimmed_Mean}, Bulyan~\cite{Bulyan}, and Flame~\cite{flame}. Please refer to the Appendix for details.

\noindent\textbf{Evaluation Metrics} ~
Based on the characteristics of federated learning and backdoor attacks, we consider the following metrics for evaluating the effectiveness of backdoor attacks and defense techniques.
\setlist[itemize]{leftmargin=*}
\begin{itemize}
\item \textbf{Backdoor Accuracy (BA)} refers to the model accuracy for backdoor tasks, where the attacker's target is to maximize BA, and an effective defense algorithm minimizes it.
\item \textbf{Main Accuracy (MA)} refers to the model accuracy for the primary task. Note that both attackers and defenders aim to minimize the attack's impact on MA.
\end{itemize}
\subsection{Results}
As shown in Table~\ref{tab:cs_am}, compared to SOTA defenses: FoolsGold, RFA, DP, Krum, Trimmed Mean, Bulyan and Flame, FLD is the most effective on all 4 datasets facing Attack A-M of constrain-and-scale~\cite{howtobackdoor}. On the MNIST dataset, RFA, Krum, Bulyan, Flame, and FLD perform fairly well, while FoolsGold, DP, and Trimmed Mean have worse performance with the Backdoor Accuracy getting close to 100\%. However, on CIFAR and Tiny-imagenet datasets, RFA and Bulyan also failed at defending against backdoors. Krum albeit successfully defended against the backdoor attack, performed significantly reduced MA. This is because CIFAR and Tiny-imagenet use ResNets neural networks, which are much more complex than the 2-layer convolutional networks deployed on MNIST. In the complex and larger models, the backdoor is hidden more deeply and thus difficult to detect or defuse with these model-granularity defensive methods. Flame has slightly better performance but is still outperformed by FLD. The reason mainly is that FLD is based on layer granularity and thus allows for fine-grained detection of anomalous models, resulting in the best defense performances on all datasets. Specifically, FLD presents 88.97\% MA on CIFAR, 25.26\% MA on Tiny-imagenet and 0.00\% BA on both datasets. In the NLP task, FLD effectively defended against backdoor attacks, with a BA of 0.00\% and MA of 19.26\% on Reddit dataset.

To further demonstrate the effectiveness of FLD, we tested FLD against several other defenses on the DBA attack, which breaks down global trigger patterns into separate local patterns and embeds them separately into the training sets of different adversaries. Compared to centralized backdoor attacks, DBA take better advantage of the distributed nature of federated learning and are therefore more stealthy and difficult to detect for federated learning.

As shown in Table~\ref{tab:dba}, RFA and Bulyan cannot effectively defend against DBA attack on MNIST, while Krum can effectively defend against the DBA attack but at the cost of a considerable drop on MA, with an accuracy of 54.12\% on CIFAR10 and 9.35\% on Tiny-imagenet. Flame achieves better defense effects on the MNIST and Tiny-imagenet, but worse performance on CIFAR10. In a word, SOTA defense methods either fail to effectively defend against DBA or endure a huge impact on the accuracy of the primary task, whereas FLD achieves effective defense on MNIST, CIFAR and Tiny-imagenet with backdoor success rates of only 0.03\%, 0.60\%, and 0.00\%, respectively, with negligible drops in MA, i.e., 0.15\%, 3.18\%, and 0.10\%.

\begin{table}[htbp]
  \centering
  \caption{Effectiveness of FLD in comparison to state-of-the-art defenses for constrain-and-scale attack, in terms of Backdoor Accuracy (BA) and primary task Accuracy (MA). All values are percentages.}
  \resizebox{\columnwidth}{!}{
    \begin{tabular}{l|cc|cc|cc|cc}
    \hline
    \multicolumn{1}{c|}{\multirow{2}[4]{*}{Defenses}} & \multicolumn{2}{|c|}{Reddit} & \multicolumn{2}{c|}{MNIST} & \multicolumn{2}{c|}{CIFAR10} & \multicolumn{2}{c}{Tiny-imagenet} \bigstrut\\
\cline{2-9}          & \multicolumn{1}{c|}{BA} & \multicolumn{1}{c|}{MA} & \multicolumn{1}{c|}{BA} & \multicolumn{1}{c|}{MA} & \multicolumn{1}{c|}{BA} & \multicolumn{1}{c|}{MA} & \multicolumn{1}{c|}{BA} & \multicolumn{1}{c}{MA} \bigstrut\\
    \hline
    No attack &   -    &    19.38   &    -   & 99.06  &    -   & 89.60  &     -  & 25.34  \bigstrut\\
   No defense & 99.63  & 19.43  & 98.01  & 98.94  & 96.70  & 88.04  & 99.41  & 25.33  \bigstrut\\
    \hline
    \hline
    FoolsGold & 100.00  & 19.16  & 99.07  & 99.01  & 98.97  & 85.24  & 99.39  & 25.47  \bigstrut[t]\\
    DP    & 99.72  & \textbf{19.41}  & 98.12  & 98.91  & 98.36  & 85.74 & 99.68  & 25.75  \\
    RFA & 100.00  & 19.44  & 0.23  & 98.98  & 86.15  & 86.62  & 0.62  & 25.21  \\
    Trimmed Mean & \textbf{0.00}  & 19.21  & 96.56  & 98.96  & 96.61  & 88.45  & 99.31  & \textbf{25.79}  \\
    Krum  & 100.00  & 19.24  & 0.17  & 97.46  & 13.94  & 56.82  & \textbf{0.00}  & 9.04  \\
    Bulyan & \textbf{0.00}  & 19.24  & 0.85  & 99.08  & 91.38  & 88.55  & 99.01  & 25.49 \\
    Flame  & \textbf{0.00}  & 19.25  & \textbf{0.00}  & 98.46  & 7.02  & 88.85  & \textbf{0.00}  & 25.56  \bigstrut[b] \\
    \hline
    \textbf{FLD} & \textbf{0.00}  & 19.26  & \textbf{0.00}  & \textbf{99.09}  & \textbf{0.00}  & \textbf{88.97}  & \textbf{0.00}  & 25.29  \bigstrut\\
    \hline
    \end{tabular}%
    }
  \label{tab:cs_am}%
\end{table}%

\begin{table}[htbp]
  \centering
  \caption{Effectiveness of FLD in comparison to SOTA defense methods against DBA attack, in terms of Backdoor Accuracy (BA) and primary task Accuracy (MA). All values are percentages.}
  \resizebox{\columnwidth}{!}{
    \begin{tabular}{l|cc|cc|cc}
    \hline
    \multicolumn{1}{c|}{\multirow{2}[4]{*}{Defenses}} &  \multicolumn{2}{|c|}{MNIST} & \multicolumn{2}{c|}{CIFAR10} & \multicolumn{2}{c}{Tiny-imagenet} \bigstrut\\
\cline{2-7}          & \multicolumn{1}{c|}{BA} & \multicolumn{1}{c|}{MA} & \multicolumn{1}{c|}{BA} & \multicolumn{1}{c|}{MA} & \multicolumn{1}{c|}{BA} & \multicolumn{1}{c}{MA} \bigstrut\\
    \hline
    No attack &     -  & 99.06  &     -  & 89.60  &     -  & 25.34  \bigstrut\\
    No defense & 99.84  & 98.94  & 97.27  & 85.55  & 99.28  & 25.64  \bigstrut\\
    \hline
    \hline
    FoolsGold & 99.81  & 98.91  & 97.27  & 84.66  & 99.34  & 25.62  \bigstrut[t]\\
    DP    & 99.59  & 98.23  & 97.46  & 83.87  & 99.58  & \textbf{25.89}  \\
    RFA & 99.83  & 98.88  & 94.17  & \textbf{87.44}  & 0.46  & 25.11  \\
    Trimmed Mean & 99.77  & 98.84  & 97.07  & 87.14  & 99.27  & 25.75  \\
    Krum  & 0.04  & 98.72  & \textbf{0.00}  & 54.12  & \textbf{0.00}  & 9.35  \\
    Bulyan & 99.83  & 98.98  & 95.83  & 86.91  & 98.82  & 25.06   \\
    Flame  & 0.15  & \textbf{98.99}  & 12.31  & 85.43  & 0.26  & 25.13  \bigstrut[b] \\
    \hline
    \textbf{FLD} &  \textbf{0.03}  & 98.91  & 0.60  & 86.42  & \textbf{0.00}  & 25.24  \bigstrut\\
    \hline
    \end{tabular}%
    }
  \label{tab:dba}%
\end{table}%

\noindent\textbf{Generalizability}~
To demonstrate the generalizability of FLD, we extend our evaluation to various backdoors such as constrain-and-scale, DBA, Edge-Case~\cite{tails}, Little Is Enough~\cite{littleisenough}, PGD~\cite{tails} and Flip attack~\cite{flip}, on CIFAR10. As summarized in Table~\ref{tab:attack}, FLD effectively mitigated all the attacks with negligible impact on MA. Therefore, FLD achieves robust performances in the face of various attacks and thus shows great generalizability.
It is worth clarifying that the target of Flip attack is to decrease the accuracy of specific labels, so the MA of Flip refers to the accuracy of the attacked class. In our experiments, we set the attacked class to be ``airplane'', and the attacker succeeded to decrease the prediction accuracy of ``airplane'' to 0.5\%. FLD effectively defended against Flip attack and present the prediction accuracy of ``airplane'' at 96.5\%.

\begin{table}[htbp]
  \centering
  \caption{FLD defends against SOTA attacks in terms of Backdoor Accuracy (BA) and primary task Accuracy (MA) on CIFAR10. All values are percentages.}
    \small 
    \begin{tabular}{l|cc|cc}
    \hline
    \multicolumn{1}{c|}{\multirow{2}[4]{*}{Attack}} & \multicolumn{2}{c|}{No Denfense } & \multicolumn{2}{c}{FLD} \bigstrut\\
\cline{2-5}          & \multicolumn{1}{c|}{BA} & MA    & \multicolumn{1}{c|}{BA} & MA \bigstrut\\
    \hline
    constrain-and-scale (A-M) & 96.70  & 88.04  & 0.00  & 88.97  \bigstrut\\
\cline{1-1}    DBA   & 97.27  & 85.55  & 0.60  & 86.42  \bigstrut\\
\cline{1-1}    Edge-Case & 76.02  & 88.79  & 7.14  & 88.11  \bigstrut\\
\cline{1-1}    Little Is Enough & 91.02  & 87.79  & 0.00  & 89.03  \bigstrut\\
\cline{1-1}    PGD   & 91.54  & 87.50  & 0.00  & 89.01  \bigstrut\\
\cline{1-1}    Neurotoxin   &    86.42   & 88.56  &    0.00   & 89.08  \bigstrut\\
    \hline
    \end{tabular}%
  \label{tab:attack}%
\end{table}%

\subsection{Ablation Tests}

\textbf{Proportion of Compromised Clients.} We use $PMR$ to represent the proportion of compromised clients, $PMR=\frac{k}{n}$, where $k$ represents the number of compromised clients and $n$ represents the number of all clients. We evaluate FLD for different $PMR$ values, i.e., 0, 0.1, 0.2, 0.3, 0.4. An effective backdoor defense  method should have minimum impact on the primary task in an environment with different proportions of attackers, also when no attackers exist. 

We followed the DBA~\cite{xie2020dba} setup and randomly selected 10 clients for each round of federated learning training, where 0, 1, 2, 3, 4 clients were malicious clients representing $PMR$ values 0, 0.1, 0.2, 0.3, 0.4, respectively. The results of the 10 rounds of the constrain-and-scale attack are shown in Figure~\ref{fig:pmr}. In CIFAR10, the success rates of BA after 10 rounds of attacks were 69.05\%, 82.24\%, 83.88\% and 85.54\%. Higher $PMR$ leads to a faster BA growth rate. FLD can successfully identify the poisoned clients in scenarios with $PMR$ of 0.1, 0.2, 0.3, and 0.4. On MNIST, when $PMR=0.1$, A-M attack needs more rounds to successfully inject the backdoor, hence the BA with No Defense is also 0. With $PMR$ of 0.2, 0.3, and 0.4, the backdoor was embedded smoothly. Nevertheless, FLD still presented a good defense against malicious clients. With $PMR=0$, FLD has a negligible effect on MA. Therefore, FLD is robust to different poisoning rates (proportion of compromised clients).
\begin{figure}[H]
\setlength{\belowcaptionskip}{-6pt} 
\centering
\includegraphics[width=\linewidth]{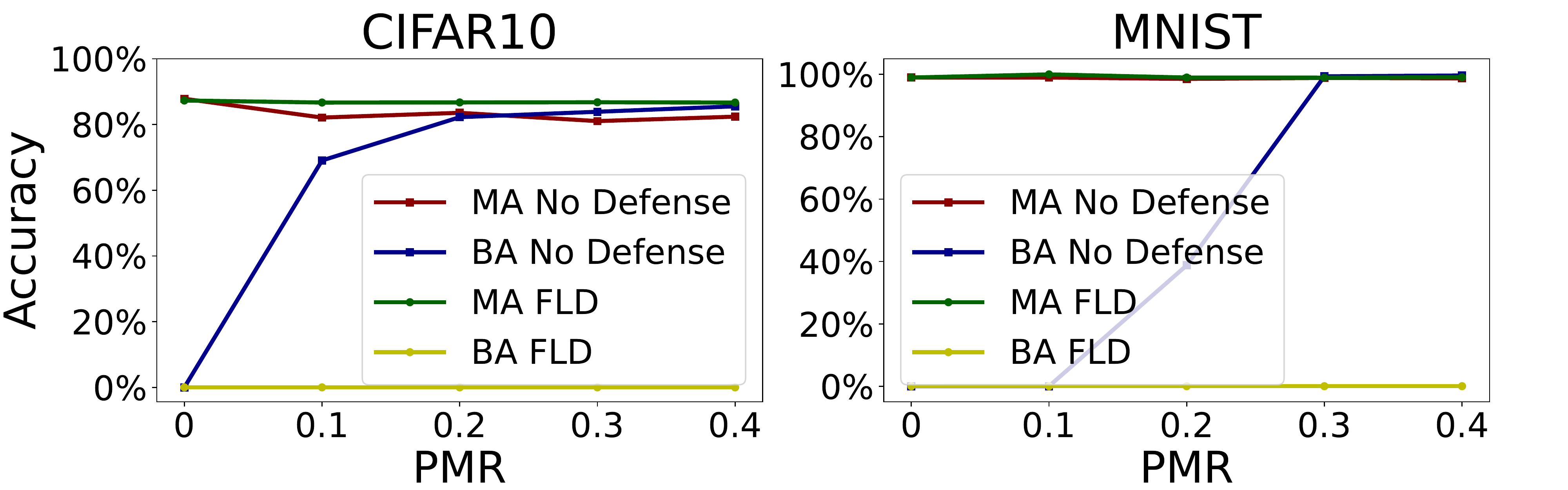}
\caption{Impact of PMR}
\label{fig:pmr}
\end{figure}
\noindent\textbf{Degree of Data Heterogeneity}. Since FLD is based on detecting the difference between benign and backdoor models, the different data distribution among clients may affect its effectiveness. To test FLD under diverse data distribution assumptions, we vary the degree of data heterogeneity across clients by adjusting the parameter of Dirichlet distribution to 0.1, 0.5, 1, 10, and 100, on the CIFAR10 dataset. Figure~\ref{fig:data} showcases the data distribution of clients when the Dirichlet alpha is 0.1 and 100. It can be seen that when Dirichlet alpha = 0.1, the data distribution is extremely heterogeneous between clients, with each client having a different number of samples and labels. When Dirichlet alpha = 100, the data distribution is close to i.i.d.
As shown in Fig.~\ref{fig:alpha} , MA decreases as the degree of data heterogeneity increases. FLD is effective in identifying and removing malicious clients in both extreme non-i.i.d and near-i.i.d scenarios.

\begin{figure}[htbp]
\centering
\includegraphics[width=\linewidth]{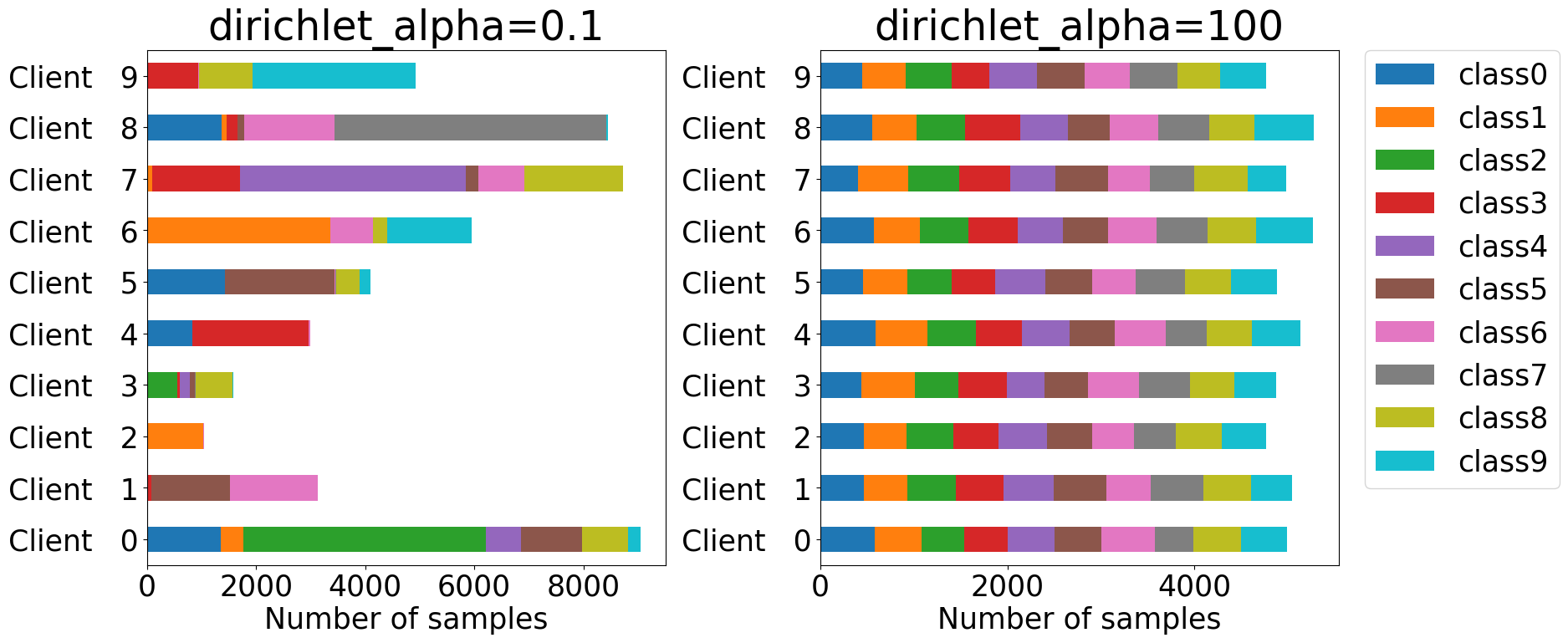}
\caption{A visualization  of client data distribution for Dirichlet alpha = 0.1 and Dirichlet alpha = 100.}
\label{fig:data}
\end{figure}

\section{Conclusions}
In this work, we show that the existing defense methods can not effectively defend against SOTA backdoor attacks in federated learning. To tackle this challenge, we propose an innovative robust federated learning backdoor defense algorithm, Federated Layer Detection (FLD). FLD is the first algorithm that assesses model outliers based on the granularity of layer. We theoretically prove the convergence guarantee of FLD on both i.i.d and non-i.i.d data distributions. Extensive experimental results prove the superiority of FLD over SOTA defense methods against various attacks in different scenarios, demonstrating the robustness and generalizability of FLD.

\bibliographystyle{IEEEtran}
\bibliography{fgcs}

\clearpage
\appendix
\setcounter{equation}{0}
\renewcommand\theequation{\arabic{equation}}
\subsection{Hyperparameters}
Reddit datasets are collected from distributed clients and therefore do not require manual division~\cite{howtobackdoor}. We use Dirichlet distribution to partition the image datasets (MNIST, CIFAR10, Tiny-imagenet). The distribution hyperparameter is 0.5, 0.9, and 0.5 for MNIST, CIFAR10, and Tiny-imagenet. Each client uses SGD as the optimizer and a default batch size of 64. $\mu $ is set as 3 by default.
For the semantic backdoor, we follow the experimental setup used by~\cite{howtobackdoor}, where the trigger sentence is ``pasta from Astoria is'' and the target word is ``delicious''. After the model was trained for 5000 rounds with 100 randomly selected clients in each round, the adversary used 10 malicious clients to inject the backdoor.
For the pixel pattern backdoor, we set specific pixels (same as the ones selected by DBA) to white, and then modify the label of the sample with the trigger to backdoor label. The backdoor label is ``digit 2'' in MNIST, ``bird'' in CIFAR10, and ``bullfrog'' in Tiny-imagenet. The default $PMR$ is $20/64$ for MNIST, $10/64$ for CIFAR, and $20/64$ for Tiny-imagenet, to be consistent with DBA~\cite{xie2020dba}. All participants train the global model, 10 of which are selected in each round to submit local SGD updates for aggregation. The adversary used 2 malicious clients to inject a backdoor in  Attack A-M of constrain-and-scale, 4 malicious clients to inject a backdoor in Attack DBA, and 1 malicious client to inject a backdoor in  Attack A-S.

\subsection{Baselines}
\setlist[itemize]{leftmargin=*}
\begin{itemize}
    \item \textbf{FoolsGold} argues that in federated learning, malicious clients tend to upload updates with higher similarities than benign clients, since each benign client has a unique data distribution while malicious clients share the same target. FoolsGold leverages this assumption to adapt the learning rate of each client during each iteration. The objective is to preserve the learning rates of the clients uploading distinct gradient updates, while decreasing the learning rates of the clients that consistently contribute similar gradient updates.
    \item \textbf{Robust Federated Aggregation (RFA)} replaces the weighted arithmetic mean with the geometric median for federated learning aggregation. Specifically, RFA aggregates the local model parameters by finding the point that minimizes the sum of the distances to all the other points, where the distance is measured using a suitable metric such as the Euclidean distance. This point is known as the geometric median, representing the ``center'' of the distribution of the local models.
    \item \textbf{Differential Privacy (DP)} is a privacy technique designed to ensure that the output does not reveal individual data records of participants.
    DP can be applied to machine learning to protect the privacy of training data or model updates. DP-based backdoor defense mitigates the impact of poisoned model updates on the global model by adding random noise to the uploaded parameters during aggregation. The random noise dilutes the malicious information injected by the malicious participants to mitigate their impact on the final global model.
    \item \textbf{Krum} selects one of the $n$ local models that is similar to the others as the global model by calculating the Euclidean distance between two of the local models.
    \item \textbf{Trimmed Mean}, also known as truncated mean, is a statistical method for calculating the average of a dataset while eliminating outliers. To compute the trimmed mean, a certain percentage of the highest and lowest values are removed or trimmed, and the mean is then calculated based on the remaining values. Specifically, trimmed mean aggregates each model parameter independently. The server ranks the $j$-th parameters of the $n$ local models. The largest and smallest $k$ parameters are removed, and the average of the remaining $n-2k$ parameters are calculated as the $j$-th parameters of the global model.
    \item \textbf{Bulyan} first iteratively applies Krum to select the local models, then aggregates these local models using a variant of the trimmed mean. In other words, Bulyan is a combination of Krum and Trimmed Mean.
    \item \textbf{Flame} 
    employs the cosine distance metric to measure the dissimilarity between locally uploaded models from various clients. It leverages HDBSCAN density clustering to distinguish between benign and malicious clients. Subsequently, Flame applies a clipping operation to limit the extent of modifications to the local models and introduces noise into the global model.
    
\end{itemize}

\subsection{Effectiveness of Layer Scoring}
To test the necessity and effectiveness of Layer Scoring, we compared the performance between FLD with and without Layer Scoring under different attacker's poisoned data rate~(PDR) on the CIFAR10 dataset. FLD without Layer Scoring uses the same approach as the existing defense methods, i.e., splicing the model weights directly and using COF and Anomaly Detection to detect anomalous models.

As shown in Figure~\ref{fig:Layer Scoring}, as the PDR increases, the BA of the FLD without Layer Scoring rises and then falls. This is due to the fact that higher PDR leads to faster backdoor injection. But, higher PDR also makes the attacks easier to detect. Hence, the impact of the attack increases in the beginning and then decreases as the defense algorithm starts to identify the attack.

\begin{figure}[htbp]
  \centering
  \begin{subfigure}[b]{0.49\linewidth}
    \centering
    \includegraphics[width=\linewidth]{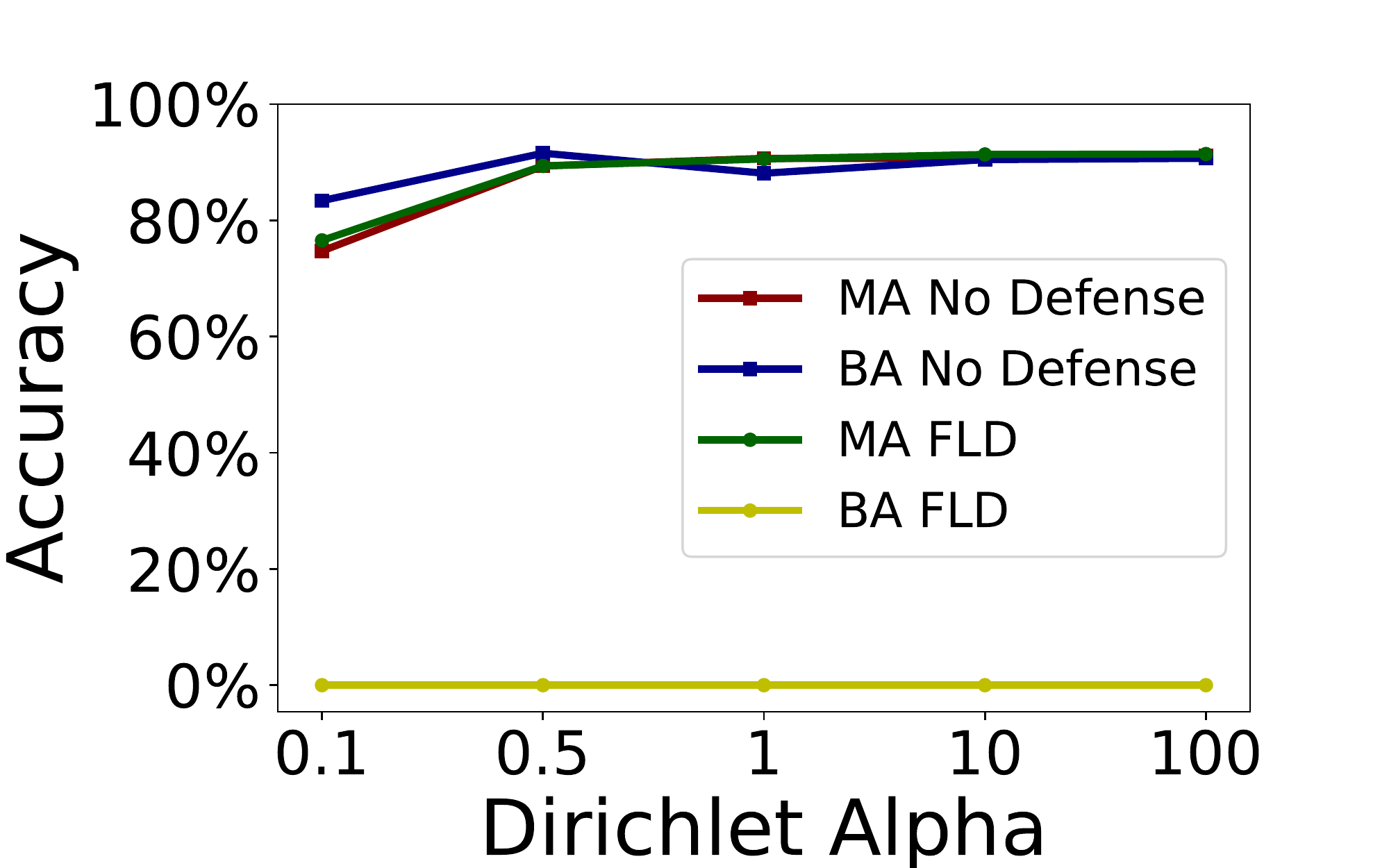}
    \caption{Impact of  data distribution}
    \label{fig:alpha}
  \end{subfigure}
  \hfill
  \begin{subfigure}[b]{0.49\linewidth}
    \centering
    \includegraphics[width=\linewidth]{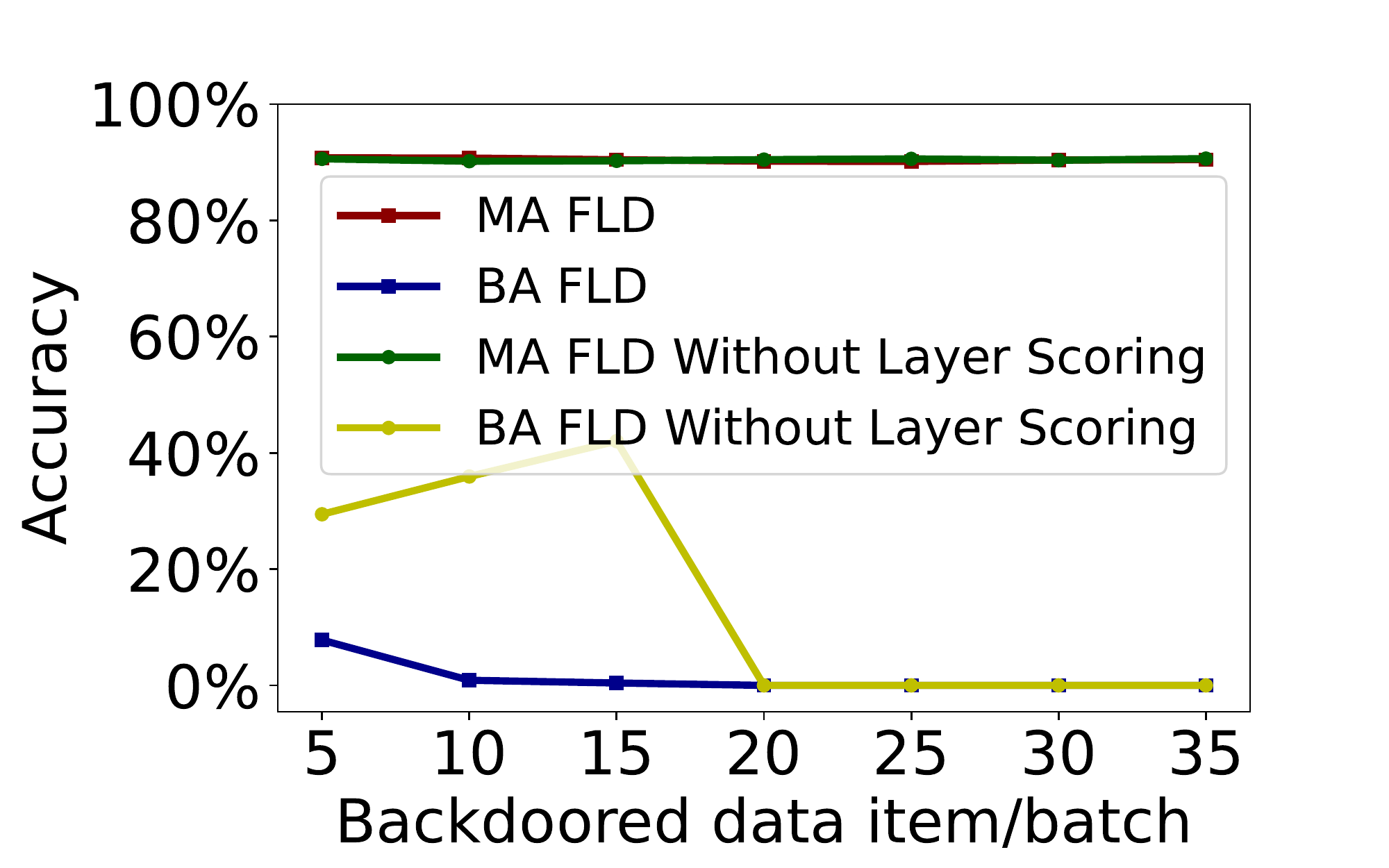}
    \caption{Impact of Layer Scoring}
    \label{fig:Layer Scoring}
  \end{subfigure}
  \caption{Comparison of different impacts}
  \label{fig:comparison}
\end{figure}

\section{Effectiveness of FLD against \textit{Attack A-S} }
In Section~Experiments, we focus on the more insidious and difficult-to-detect \textit{Attack A-M}. Here we evaluate the effectiveness of FLD against \textit{Attack A-S}. Table~\ref{tab:as} shows that FLD is effective for all 4 datasets on \textit{Attack A-S} of constrain-and-scale. 

\begin{table}[htbp]
  \centering
  \caption{Effectiveness of FLD in comparison to state-of-the-art defenses for constrain-and-scale  \textit{Attack A-S}, in terms of Backdoor Accuracy (BA) and Main Task Accuracy (MA). All values are percentages.}
  \resizebox{\columnwidth}{!}{
    \begin{tabular}{l|cc|cc|cc|cc}
    \hline
    \multicolumn{1}{c|}{\multirow{2}[4]{*}{Defenses}} & \multicolumn{2}{c|}{Reddit} & \multicolumn{2}{c|}{MNIST} & \multicolumn{2}{c|}{CIFAR10} & \multicolumn{2}{c}{Tiny-Imagenet} \bigstrut\\
\cline{2-9}          & \multicolumn{1}{c|}{BA} & \multicolumn{1}{c|}{MA} & \multicolumn{1}{c|}{BA} & \multicolumn{1}{c|}{MA} & \multicolumn{1}{c|}{BA} & \multicolumn{1}{c|}{MA} & \multicolumn{1}{c|}{BA} & \multicolumn{1}{c}{MA} \bigstrut\\
    \hline
    No attack &       & 19.38  &       & 99.06  &       & 89.60  &       & 25.34  \bigstrut\\
\cline{1-1}    No defense & 100.00  & 19.35  & 70.01  & 48.03  & 78.86  & 60.60  & 99.19  & 20.84  \bigstrut\\
    \hline
    \hline
    FoolsGold & 100.00  & 19.35  & 70.06  & 48.04  & 53.47  & 77.31  & 99.28  & 21.02  \bigstrut[t]\\
    DP    & 100.00  & 19.35  & 70.01  & 47.88  & 78.07  & 61.11  & 99.31  & 20.58  \\
    RFA & 100.00  & 19.40  & \textbf{0.01}  & 98.78  & \textbf{0.00} & \textbf{89.28}  & \textbf{0.00}  & 25.41  \\
    Trimmed Mean & \textbf{0.00}  & 19.40  & 0.08  & 98.73  & \textbf{0.00}  & 88.57  & 0.21  & \textbf{25.61}  \\
    Krum  &    \textbf{0.00}   &   14.51    & 0.02  & \textbf{98.85}  & \textbf{0.00}  & 44.50  & \textbf{0.00}  & 7.76  \\
    Bulyan & \textbf{0.00}  & 19.39  & 0.13  & 98.67  & \textbf{0.00}  & 88.75  & \textbf{0.00}  & 25.47  \bigstrut[b]\\
    \hline
    \textbf{FLD} & \textbf{0.00}  & \textbf{19.43}  & 0.04  & \textbf{98.85}  & \textbf{0.00}  & 89.26  & \textbf{0.00}  & 25.34  \bigstrut\\
    \hline
    \end{tabular}%
    }
  \label{tab:as}%
\end{table}%

\end{document}